\def\year{2021}\relax
\documentclass[letterpaper]{article} 
\usepackage{aaai21}  
\usepackage{times}  
\usepackage{helvet} 
\usepackage{courier}  
\usepackage[hyphens]{url}  
\usepackage{graphicx} 
\usepackage{natbib}  
\usepackage{caption} 
\usepackage[switch]{lineno}

\newcommand{\ignore}[1]{}

\usepackage[ruled, vlined, linesnumbered]{algorithm2e}
\usepackage{amsthm}
\usepackage{amsmath,amssymb}
\usepackage{mathtools}
\usepackage{caption, subcaption}
\DeclareCaptionType{copyrightbox}
\usepackage{bm}
\usepackage{natbib}
\usepackage{dsfont}
\usepackage{booktabs}

\usepackage[table, dvipsnames]{xcolor}

\usepackage{array}
\newcolumntype{P}[1]{>{\centering\arraybackslash}p{#1}}
\newcolumntype{M}[1]{>{\centering\arraybackslash}m{#1}}

\def\shownotes{1}  
\ifnum\shownotes=1
\newcommand{\authnote}[2]{{$\ll$\textsf{#1 notes: #2}$\gg$}}
\else
\newcommand{\authnote}[2]{}
\fi

\newtheorem{theorem}{Theorem}
\newtheorem{lemma}[theorem]{Lemma}

\newtheorem{proposition}[theorem]{Proposition}

\newtheorem{remark}[theorem]{Remark}

\newtheorem{assumption}[theorem]{Assumption}

\let\hat\widehat
\let\tilde\widetilde
\let\leq\leqslant
\let\geq\geqslant

\newcommand{\ind}[1]{\mathds{1}_{#1}}

\newcommand{\E}{\mathbb{E}}
\let\Pr\relax
\DeclareMathOperator{\Pr}{\mathbb{P}}

\newcommand{\cond}{~|~}

\newcommand{\eqdef}{\mathbin{\stackrel{\rm def}{=}}}
\newcommand{\diff}{\mathrm{d}}
\newcommand{\reg}{\mathfrak{R}}
\newcommand{\cvar}{\mathrm{CVaR}}
\newcommand{\var}{\mathrm{VaR}}

\DeclareMathOperator{\argmax}{\mathrm{argmax}}

\usepackage{xcolor}
\usepackage[normalem]{ulem}
\makeatletter
\def\uwave{\bgroup \markoverwith{\lower3.5\p@\hbox{\sixly \textcolor{blue}{\char58}}}\ULon}
\font\sixly=lasy6 
\makeatother

\title{Near-Optimal MNL Bandits Under Risk Criteria}
\author {
        Guangyu Xi, \textsuperscript{\rm 1}
        Chao Tao, \textsuperscript{\rm 2}
        Yuan Zhou \textsuperscript{\rm 3} \\
}
\affiliations {
    \textsuperscript{\rm 1} University of Maryland, College Park \\
    \textsuperscript{\rm 2} Indiana University Bloomington \\
    \textsuperscript{\rm 3} University of Illinois at Urbana-Champaign \\
    gxi@umd.edu, taochao@iu.edu, yuanz@illinois.edu
}

\begin{document}

\maketitle

\begin{abstract}
We study MNL bandits, which is a variant of the traditional multi-armed bandit problem, under risk criteria. Unlike the ordinary expected revenue, risk criteria are more general goals widely used in industries and business. We design algorithms for a broad class of risk criteria, including but not limited to the well-known conditional value-at-risk, Sharpe ratio, and entropy risk, and prove that they suffer a near-optimal regret. As a complement, we also conduct experiments with both synthetic and real data to show the empirical performance of our proposed algorithms.
\end{abstract}

\section{Introduction}
Dynamic assortment optimization is one of the fundamental problems in online learning. It has wide applications in industries, for example retailing and advertisement. To motivate the study of the problem, let us consider e-commerce companies like Amazon and Wish who want to sell products to online users when they visit the websites and search for some type of products, for example, headphones. Such companies usually have a variety of products with that type in a warehouse to sell. Due to the space constraints of a website, it is not possible to exhibit all of the available products. Hence, each time when an online user visits the website, only a limited number of products can be displayed. When an online user buys a product, the company will get some profit. So one natural goal for the company is to display on the website an assortment consisting of several products such that the expected revenue is maximized. However, in practice, a company may have more complex strategies other than simply maximizing its revenue, and general \textit{risk criteria} may be better choices to serve such goals. For example, in risk management, a very common risk criterion called \textit{expected shortfall} or \textit{conditional value at risk} (CVaR) is defined as the expected revenue under a certain percentile. If we only consider the expected revenue, we may lead to focus on recommending some products producing high revenue but purchased only by a small portion of users. If the company wishes to maintain a higher level of active and diversified users, then CVaR is more appropriate. Whether it is still possible for the sales manager of the company to design a near-optimal sales strategy when the goal is changed, for example to a kind of risk criteria, is a very practical problem and to the best of our knowledge, has not been studied before.

Suppose a company has $N$ products of a certain category to sell during a sales season, which can be represented by $[N]$, where $[N] \eqdef \left\{ 1,2,\dots,N \right\} $ and each product corresponds to an element in $[N]$. Let $T$ be the total number of times such products are searched during a sales season and $S_t$ be the assortment displayed by the website at $t$th time of the request. The aforementioned sales activity can be modeled by the following game which runs in $T$ time steps: at time step $1 \leq t \leq T$, when an assortment $S_t \subset [N]$ is displayed by the website, the online user will make a choice, i.e., whether to buy a product in $S_t$ or purchase nothing. Following the previous motivation example, we add a cardinality constraint, which means the number of products in $S_t$ can not exceed a predefined number $K \leq N$. Let $c_t$ denote the choice of the online user at time $t$. When $c_t = i$, it means that the online user buys product $i$. For convenience, we use $c_t = 0$ to represent the situation when the online user does not purchase anything. In general, $c_t$ can be viewed as a random variable and there is no doubt that the \textit{multinomial logit model} (MNL) \cite{agrawal2019mnl} has become the most popular one to model the behavior of the online user, i.e., $c_t$, when $S_t$ is provided. Dynamic assortment optimization with the MNL choice model is also called MNL bandits. In this model, each product $i$ is assumed to be related to an \textit{unknown} preference parameter $v_i$ and the probability that a visiting online user chooses product $i$ given assortment $S_t$ is defined by
\begin{equation} \label{model}
\Pr( c_t = i ) \eqdef \frac{ v_i }{1 + \sum_{ j \in S_t } v_j},
\end{equation}
where we set the preference parameter of no-purchase $v_0 = 1$. Note that this assumption does not harm the model too much since one can easily scale $v_i$'s to satisfy this condition. Following the literature, we also assume no-purchase is the most frequent choice i.e., $0 \leq v_i \leq 1$, which is often a reasonable assumption in sales activities.

During the last decade, MNL bandits has attracted much attention \cite{DBLP:journals/ior/RusmevichientongSS10, DBLP:journals/msom/SaureZ13, DBLP:conf/colt/AgrawalAGZ17, agrawal2019mnl, DBLP:journals/corr/abs-2007-04876}. However, all of the previous works consider maximizing the expected revenue, which is not always appropriate for practical applications. In this paper, we are interested in designing algorithms for a general class of risk criteria. 

\section{Problem Formulation}
Suppose for each product $i \in [N]$, selling it successfully can make the company a profit of $r_i$, which is \textit{known} beforehand. Without loss of generality, we assume  $r_i \in (0, 1]$. This can always be achieved by proper scaling. Moreover, the profit for no-purchase is $r_0=0$. Then at time step $t\geq 1$, when assortment $S_t\subset [N]$ and preference parameter vector $ \bm{v} = (v_1, \dots, v_N) $ are provided, the profit can be represented by a random variable $X(S_t,\bm{v})$ defined by
\begin{equation}
    \Pr (X(S_t,\bm{v})=r_i)=\Pr( c_t = i ) =\frac{ v_i }{1 + \sum_{ j \in S_t } v_j}
\end{equation}
for $i=0$ and all $i\in S_t$. In addition, we denote by $F(S_t,\bm{v})$ the cumulative distribution function to $X(S_t,\bm{v})$. Given time horizon $T$, one natural goal, as explained in the introduction, is to find a policy equipped by the decision maker such that the expected revenue, i.e., 
\begin{equation}\label{ExpectedRev}
    \sum_{t=1}^T R (S_t ,\bm{v})= \sum_{t=1}^T \mathbb{E}\left[ X(S_t ,\bm{v})\right]
\end{equation}
is maximized, where $R(S_t,\bm{v})$ represents the expected profit when $S_t$ is served. This has been investigated previously in \cite{DBLP:conf/colt/AgrawalAGZ17, agrawal2019mnl}.

In this paper, instead of expectation, we consider a general class of risk criteria. Some examples of such risk criteria can be found in \cite{cassel2018general}. Suppose $\mathcal{D}$ is the convex set of cumulative distribution functions. In general, we consider the risk criterion $U$ which is a function from $\mathcal{D}$ to $\mathbb{R}$. In the case of expectation, $U(F)=\int x \diff F(x)$. In particular, since we assumed that $r_i\in(0,1]$, we will only need $F\in \mathcal{D}[0,1]$, where we denote by $\mathcal{D}[0,1]$ the subspace of $\mathcal{D}$ consisting of $F$ that is the cumulative distribution function of  random variable $X$ taking values on $[0,1]$. 
The goal of this paper is to find a policy such that the following quantity
\begin{equation} \label{eq:goal}
\E \left[ \sum_{t = 1}^T U(F(S_t,\bm{v})) \right ]
\end{equation}
is maximized. Let $S^*$ be the smallest assortment such that 
\begin{equation*}
U( F(S^*,\bm{v})) = \max_{S \subset [N], |S| \leq K} U(F(S,\bm{v})).
\end{equation*}
The regret of the game after $T$ time steps, which is a quantity measuring the difference between the optimal policy and policy $\pi$ used by the decision maker, is defined as
\begin{multline}  \label{eq:regret}
\reg_T^{\pi}([N], \bm{v}, \bm{r}) \\
\eqdef T U( F(S^*,\bm{v})) - \E\left[  \sum_{t = 1}^T U(F(S_t,\bm{v})) \right]. 
\end{multline}
where $\bm{r}=(r_1,\cdots,r_N)$ and  $\bm{v}=(v_1,\cdots,v_N)$.

When it is clear from the context, we usually omit the policy $\pi$ and parameters $([N], \bm{v}, \bm{r})$. Without much effort, we can see that maximizing \eqref{eq:goal} is equivalent to minimizing the regret \eqref{eq:regret}.

\section{Related Work and Our Contribution}
To the best of our knowledge, we are the first to study MNL bandits under general risk criteria. 

In the past decade, there have been many works on the MNL bandit problem considering maximizing the expected revenue \eqref{ExpectedRev}. In \cite{DBLP:journals/ior/RusmevichientongSS10, DBLP:journals/msom/SaureZ13}, the authors assumed the gap between the best and the second-to-the-best assortments is known and proposed ``Explore-then-Commit'' algorithms. Later in \cite{agrawal2019mnl}, the authors proposed the state-of-the-art UCB-type algorithm with a regret upper bound $\mathcal{O}( \sqrt{NT \ln T} )$. Authors in \cite{DBLP:conf/colt/AgrawalAGZ17} utilized the Bayesian method i.e., Thompson Sampling to design an algorithm which performs well in practice. For the expected revenue, it is showed in \cite{chen2018note} that the lower bound for the regret is $\Omega( \sqrt{NT} )$. 

There are a lot of previous works studying different risk criteria in multi-armed bandits \cite{DBLP:conf/nips/SaniLM12, DBLP:conf/alt/Maillard13, DBLP:journals/corr/GalichetST14,  DBLP:journals/corr/ZiminIC14, DBLP:journals/jstsp/VakiliZ16}. In \cite{cassel2018general}, the authors established a thorough theory to deal with general risk criteria.

\paragraph{Our Contribution} Note that directly applying the algorithm proposed in \cite{cassel2018general} will lead to a regret of $\Omega \left( \sqrt{ {N \choose K} T } \right)$ since each assortment corresponds to an arm, which is far from being optimal. We can not simply take each product as an arm since the optimal assortment may consist of multiple products due to that there is an involved relationship between the risk criterion of the assortment and the underlying preference parameters, which is characterized by the following two complex structures: a general (and non-specific) risk function and the multinomial logit choice model. In this paper, we are able to gain a clearer understanding of the challenge raised by these two complex structures. To be more specific, we recognize three mild conditions that are easy to verify and show that the class identified by the aforementioned conditions encompasses most of the risk criteria that are of interest in literature (see Table~\ref{table-of-reward-function}). We also design and analyze the algorithmic framework, proving that our algorithm achieves $\tilde{\mathcal{O}}(\sqrt{NT})$ regret for any general risk criterion that belongs to the class.

\section{Assumptions}
In this section, we first present the aforementioned three assumptions the risk criterion $U$ should satisfy. 
\begin{assumption}[Quasiconvexity] \label{ass:quasi-convex}
$U$ is \textit{quasiconvex} on $\mathcal{D}[0, 1]$, i.e., for any $\lambda \in [0, 1]$ and $F_1,F_2\in \mathcal{D}[0, 1]$, it satisfies
\begin{equation} \label{eq:quasi-convexity}
U(\lambda F_1 + (1-\lambda) F_2 ) \leq \max\{ U(F_1), U(F_2) \}.
\end{equation}
\end{assumption}
In addition to quasiconvexity, we also make the following two assumptions on $U$.

\begin{assumption}[Boundedness] \label{ass:bounded}
For any $F \in \mathcal{D}[0, 1]$ it holds that $|U (F)| \leq \gamma_1 $.
\end{assumption}

\begin{assumption}[One-sided Lipschitz Condition] \label{ass:lipschitz}
For any $\bm{v}' \geq \bm{v}$, i.e., $v_i'\geq v_i$ for all $i\in[N]$, and $S \subset [N]$, it holds that
\begin{multline*}
U (F(S,\bm{v}'))  - U( F(S,\bm{v}))  \\ 
\leq \frac{ \gamma_2 }{1+\sum_{i \in S} v_i} \left[\sum_{i \in S} (v_i' - v_i) \right].
\end{multline*}
\end{assumption}

Note that here $\gamma_1$ and $\gamma_2$ are universal constants related to the risk criterion $U$.

\ignore{It seems restrictive to assume that the risk criteria satisfy quasiconvexity. However, we do not see any widely used risk criteria that is not quasiconvex in the literature.} We note that quasiconvexity is a natural assumption as most practical risk criteria considered in literature are quasiconvex. In Table~\ref{table-of-reward-function}, we give a list of the risk criteria considered, which are all quasiconvex as shown in \cite{cassel2018general}. To complement, we also show in Table~\ref{table-of-reward-function} that whether a risk criterion satisfies Assumption~\ref{ass:bounded} and Assumption~\ref{ass:lipschitz}, and give concrete values of $\gamma_1$ and $\gamma_2$. It turns out that all the risk criteria listed in Table~\ref{table-of-reward-function} satisfy all three assumptions except $\var$, which does not meet Assumption \ref{ass:lipschitz} since it is discontinuous in $\bm{v}$, 

\begin{table*}[t]
\centering
\begin{tabular}{| M{5cm} | M{3cm} | M{3cm}| M{3cm}| }
\toprule\addlinespace[0pt]
 Risk Criterion (Parameter) & Property &$\gamma_1$ & $\gamma_2$ \\ \hline
$ \var_{\alpha}$ & Quasiconvex & $1$  &  Not Exist \\ \hline
 $\cvar_{\alpha}$ &  Convex & $1$  &  $3/\alpha$ \\ \hline
  $n$th-moment &  Linear & $1$  &  $1$  \\ \hline
   Entropy risk ($\theta$) &  Convex & $1$  &  $2e^\theta /\theta$  \\ \hline
   Below target semi-variance ($r$) &  Linear & $r^2$  &  $2r^2$  \\ \hline
   Negative variance &  Convex & $\frac{1}{4}$  &  $6$  \\ \hline
   Mean-variance ($\rho$) &  Convex & $1+\frac{\rho}{4}$  &  $2+6\rho$  \\ \hline
   Sharpe ratio ($r,\epsilon$) &  Quasiconvex & $\frac{1}{\sqrt{\epsilon}}$  &  $2\epsilon^{-1/2}+3\epsilon^{-3/2}$  \\ \hline
   Sortino ratio ($r,\epsilon$) &  Quasiconvex & $\frac{1}{\sqrt{\epsilon}}$  &  $2\epsilon^{-1/2}+\epsilon^{-3/2}$  \\ \hline
\end{tabular}
\caption{Widely Used Risk Criteria}
\label{table-of-reward-function}
\end{table*}

\section{Algorithms}
Due to space constraints, we only show $\mathtt{RiskAwareUCB}$, which is a variant of the UCB-type algorithm proposed in \cite{agrawal2019mnl}, and its guarantee. In a similar way, we also propose $\mathtt{RiskAwareTS}$, a variant of the Thompson Sampling algorithm proposed in \cite{DBLP:conf/colt/AgrawalAGZ17}. Please refer to Appendix~B for its precise description and near-optimal guarantee.

The high level idea of the proposed algorithm $\mathtt{RiskAwareUCB}$ is as follows. We divide all the time steps i.e., $[T]$ into small episodes. During each episode $\ell$, the same assortment $S_{\ell}$ is repeatedly provided to the online user until a no-purchase outcome is observed. Specifically, in each episode $\ell$, we are providing the assortment
\begin{equation*}
     \argmax_{ S \subset [N], |S| \leq K } U (F(S,\tilde{\bm{v}}^{\ell})),
\end{equation*}
where $\tilde{\bm{v}}^{\ell}$ is an optimistic estimate of the real unknown preference parameters before the start of episode $\ell$. The details of $\mathtt{RiskAwareUCB}$ is described in Algorithm~\ref{alg:riskaware-ucb} where $t_{i, \ell}$ is the number of times the online users buy product $i$ in the $\ell$th episode and $\mathcal{T}_i(\ell)$ denotes the collection of episodes for which product $i$ is served until episode $\ell$ (exclusive).

\begin{algorithm}[t]
\DontPrintSemicolon
\caption{$\mathtt{RiskAwareUCB}(N, K, \bm{r}, U)$}
\label{alg:riskaware-ucb}
Initialize $t \leftarrow 1$, $\ell \leftarrow 1$, $\tilde{v}^{\ell}_i \leftarrow 1$ for $i \in [N]$ and $\mathcal{T}_i(\ell) \leftarrow \emptyset$ for $i \in [N]$ \\
\While {$t \leq T$}{
$S_{\ell} \leftarrow \argmax_{ S \subset [N], |S| \leq K } U (F(S,\tilde{\bm{v}}^{\ell}))$ \\
Initialize $t_{i, \ell} \leftarrow 0$ for $i \in [N]$ \\
\Repeat {$t > T$ or $c_{t-1} = 0$}{
Serve $S_{\ell}$ and observe customer choice $c_t$ \\
\lIf{$c_t \neq 0$}{$t_{c_t, \ell} \leftarrow t_{c_t, \ell} + 1$
}
$t \leftarrow t+1$
}
\For{$i \in [N]$}{
\lIf{$i \in S_{\ell}$}{$\mathcal{T}_i(\ell+1) \leftarrow \mathcal{T}_i(\ell) \cup \{ \ell \}$}
\lElse{$\mathcal{T}_i(\ell+1) \leftarrow \mathcal{T}_i(\ell)$}
}
$\ell \leftarrow \ell + 1$, $T_i(\ell) \leftarrow |\mathcal{T}_i(\ell)|$ for $i \in [N]$ \\
$\bar{v}_i^\ell \leftarrow \frac{ \sum_{\ell' \in \mathcal{T}_i(\ell)} t_{i, \ell' } }{ T_i(\ell) }$ for $i \in [N]$ \\
$\tilde{v}^{\ell}_i \leftarrow \min \Bigg \{ \bar{v}_i^\ell + \sqrt{\bar{v}_i^\ell \cdot \frac{48 \ln (\sqrt{N}\ell+1)}{T_i(\ell) }} + \frac{48 \ln (\sqrt{N}\ell+1)}{T_i(\ell) }, 1 \Bigg \}$ for $i \in [N]$
}
\end{algorithm}

\ignore{Let $t_{i, \ell}$ be the number of times the online users buy product $i$ in the $\ell$th episode and $\mathcal{T}_i(\ell)$ be the collection of episodes for which product $i$ is served until episode $\ell$ (exclusive). Define $T_i(\ell) \eqdef |\mathcal{T}_i(\ell)|$ and 
\begin{equation}
    \bar{v}_i^\ell \eqdef \frac{ \sum_{\ell' \in \mathcal{T}_i(\ell)} t_{i, \ell' } }{ T_i(\ell) } .
\end{equation}
The $i$th component of the optimistic preference parameters $\tilde{\bm{v}}^{\ell}$ is given by
\begin{multline*}
\tilde{v}^{\ell}_i = \min \Bigg \{ \bar{v}_i^\ell + \sqrt{\bar{v}_i^\ell \cdot \frac{48 \ln (\sqrt{N}\ell+1)}{T_i(\ell) }}  \\
  + \frac{48 \ln (\sqrt{N}\ell+1)}{T_i(\ell) }, 1 \Bigg \}.
\end{multline*}}

Then we have the following theoretical upper bound for $\mathtt{RiskAwareUCB}$.
\begin{theorem} \label{thm:main}
Suppose the risk criterion $U$ satisfies Assumption \ref{ass:quasi-convex}, \ref{ass:bounded} and \ref{ass:lipschitz}. The regret \eqref{eq:regret} incurred by the decision maker using $\mathtt{RiskAwareUCB}$ is upper bounded by $\tilde{\mathcal{O}} ( \sqrt{NT} )$ after $T$ time steps, where $\tilde{\mathcal{O}}$ hides poly-logarithmic factors in $N$ and $T$.
\end{theorem}

Before proceeding, we first prove the following key lemma, which says that the risk gain of the optimal assortment calculated by an optimistic estimate of the preference parameters is never worse than that of $S^*$.

\begin{lemma}[Monotone Maximum] \label{lem:monomax}
For any $\bm{v}' \geq \bm{v}$, it holds that
$$
\max_{S \subset [N],\vert S\vert \leq K} U (F(S,\bm{v}')) \geq U (F(S^*,\bm{v})).
$$
\end{lemma}

\begin{proof}
Fix $S$, we first prove that $U (F(S,\bm{u}))$ is a quasiconvex function with respect to vector $\bm{u}$. This statement can be easily verified by noticing that for any $\lambda \in [0, 1]$ and $\bm{u}'$, we have
\begin{align*}
 & U (F(S, \lambda \bm{u} + (1-\lambda) \bm{u}') )  \\
= {} &  U \bigg( \frac{ \lambda(1+\sum_{i \in S} u_i )}{ \lambda(1+\sum_{i \in S} u_i) + (1 - \lambda)(1+\sum_{i \in S} u'_i) } F(S, \bm{u}) \\
& + \frac{ (1 - \lambda)(1+\sum_{i \in S} u'_i) }{ \lambda(1+\sum_{i \in S} u_i) + (1 - \lambda)(1+\sum_{i \in S} u'_i) } F(S, \bm{u}' ) \bigg) \\
\leq {} &  \max\{ U(F(S, \bm{u} )), U(F(S, \bm{u}' )) \}
\end{align*}
where the last inequality is due to quasiconvexity of $U$ on $\mathcal{D}[0, 1]$. 

Next, we show the following lemma.

\begin{lemma} \label{lem:optimistic}
Given a quasiconvex function $V(\bm{u})$ defined on $[0,1]^n$, suppose there is a point $\bm{\bar{u}}=(\bar{u}_1,\cdots,\bar{u}_n) \in [0,1]^n$ satisfying that $V(\bm{\bar{u}}) \geq V(\bm{u})$ for any point $\bm{u} \neq \bar{ \bm{u}} $ such that $u_i=\bar{u}_i \mbox{ or } 0$ for each $i$ from $1$ to $n$. Then we have that $V(\bm{u}') \geq V(\bm{\bar{u}})$ for any $\bm{u}' \geq \bm{\bar{u}}$.
\end{lemma}

\begin{proof}
For the sequence of points $\bm{u}^{(i)}=(u^{(i)}_1,\dots,u^{(i)}_n )$ with $i=1,2,\dots,n$ such that
\begin{equation*}
u^{(i)}_{j}=
\begin{cases}
      \bar{u}_j \qquad &j\neq i \\
      0 \qquad &j=i,
\end{cases}
\end{equation*}
we have that $V(\bm{u}^{(i)}) < V(\bar{\bm{u}})$. For any $\bm{u}' \geq \bm{\bar{u}}$ and $\bm{u}' \neq \bar{ \bm{u} }$, we define
$$\lambda_i = \left( \frac{u_i' -\bar{u}_i}{\bar{u}_i} \right)\left(\sum_{i=1}^{n} \frac{u_i' -\bar{u}_i}{\bar{u}_i}\right)^{-1}.
$$
Here $\lambda_i\in[0,1]$ for all $i=1,2,\cdots,n$ and $\sum_{i=1}^n \lambda_i=1$. Then the convex combination of $\bm{u}^{(i)}$
$$\tilde{\bm{u}} =\sum_{i=1}^{n} \lambda_i \bm{u}^{(i)}
$$
is on the same line as $\bm{u}'$ and $\bm{\bar{u}}$. By the quasiconvexity of $V$, we have that $V(\bm{\tilde{u}})\leq \max _{i=1}^{n} V(\bm{u}^{(i)}) < V(\bm{\bar{u}})$. If we define  $\lambda = \frac{1 }{1 + \sum_{i = 1}^n \lambda_i } $ , then
$$V(\bm{\bar{u}}) = V(\lambda \bm{\tilde{u}} +(1-\lambda)\bm{u}')\leq \max \{V(\bm{\tilde{u}}), V(\bm{u}') \},
$$
which means we must have $V(\bar{ \bm{u} }) \leq V( \bar{ \bm{u} }' )$.
\end{proof}

Let $S^{\dag}$ be the smallest assortment such that 
\begin{equation}
    U(F(S^\dag,\bm{v}'))=\max_{S\subset [N],\vert S\vert \leq K}U(F(S,\bm{v}')).
\end{equation}
Together with Lemma~\ref{lem:optimistic}, we obtain that $U (F(S^\dag,\bm{v}')) \geq U (F(S^*,\bm{v}')) \geq U (F(S^*,\bm{v}))$, which concludes the proof of this lemma.
\end{proof}

\begin{lemma} \label{lem:conf-bound}
Given any $\ell>0$ and $C_1,C_2>0$, we define event 
\begin{multline*}
\mathcal{E}_{\ell}=\Bigg \{ \forall i \in [N], v_i \leq \tilde{v}^{\ell}_i \leq v_i + C_1  \sqrt{\frac{ v_i \ln( \sqrt{N} \ell + 1 ) }{T_i(\ell) \vee 1} } \\ + C_2 \frac{ \ln( \sqrt{N} \ell + 1 ) }{T_i(\ell) \vee 1} \Bigg \}.
\end{multline*}
There exist real numbers $C_1, C_2 > 0$ such that
\begin{equation*} \Pr( \mathcal{E}_{\ell} ) \geq 1 - \frac{1}{\ell}
\end{equation*}
for any $\ell$.
\end{lemma}
Lemma \ref{lem:conf-bound} can be easily derived from Lemma 4.1 of \cite{agrawal2019mnl}. So we omit its proof here.

\begin{proof}[Proof of Theorem~\ref{thm:main}]
Before proceeding, we introduce several notations. Let $L$ be the total number of episodes when $\mathtt{RiskAwareUCB}$ stops after $T$ steps. Denote $l_{\ell}$ by the length of the $\ell$th episode. Moreover, set $n_i=T_i(L)$, which is the total number of episodes product $i$ is served before the $L$th episode.

Using the law of total expectation, we rewrite the regret as
\begin{align*}
\reg_T & = \E\left[ \sum_{ \ell = 1}^L l_{\ell} \left( U (F(S^*,\bm{v})) - U (F(S_{\ell},\bm{v}))\right)  \right] \\
& = \E\left[ \sum_{\ell = 1}^L \E[ l_{\ell} ( U (F(S^*,\bm{v})) - U (F(S_{\ell},\bm{v})) ) \cond \mathcal{H}_{\ell} ]  \right],
\end{align*}
where $\mathcal{H}_{\ell}$ is the history before episode $\ell$. Since $S_{\ell}$ is determined by $\mathcal{H}_{\ell}$, there is
\begin{align*}
\reg_T  = \E\left[ \sum_{\ell = 1}^L \E[ l_{\ell}  \cond \mathcal{H}_{\ell} ]( U (F(S^*,\bm{v})) - U (F(S_{\ell},\bm{v})) )  \right].
\end{align*}
Given $S_{\ell}$, we know that $l_{\ell}$ follows a geometric distribution with parameter $1/(1+\sum_{i \in S_{\ell}} v_i )$. Hence we have $\E[ l_{\ell} \cond \mathcal{H}_{\ell} ] \leq 1 + \sum_{i \in S_{\ell} } v_i$. We put inequality here since the last episode may end due to time limit. Using aforementioned inequality, we further derive
\begin{align}
     \reg_T  \nonumber 
     & \leq \E \Bigg[ \sum_{\ell = 1}^L \E \Bigg[ \left(1 + \sum_{i \in S_{\ell} } v_i \right) \nonumber \\ 
    & \quad \times (  U (F(S^*,\bm{v})) - U (F(S_{\ell},\bm{v}))  ) \Bigg] \Bigg] \nonumber \\
    & = \E \left[ \sum_{\ell = 1}^L \E \delta_{\ell} \right] \label{main-thm:equ-1},
\end{align}
where we have defined $\delta_{ \ell } \eqdef  (1 + \sum_{i \in S_{\ell} } v_i) \times (  U (F(S^*,\bm{v})) - U (F(S_{\ell},\bm{v}))) $.

We now focus on bounding $\E \delta_{\ell}$. By a simple calculation, we get
\begin{align}
    \E \delta_{\ell} & = \E[ \delta_{\ell} \ind{ \mathcal{E}_{\ell}^c } ] + \E[ \delta_{\ell} \ind{\mathcal{E}_{\ell} } ] \nonumber \\
    & \leq 2 \gamma_1 (N + 1)  \Pr( \mathcal{E}_{\ell}^c ) + \E[ \delta_{\ell} \cond \mathcal{E}_{\ell} ] \Pr( \mathcal{E}_{\ell} ) \nonumber \\
    & \leq \frac{2\gamma_1(N + 1)}{ \ell } + \E[ \delta_{\ell} \cond \mathcal{E}_{\ell} ] \Pr( \mathcal{E}_{\ell} ), \label{main-thm:equ-2}
\end{align}
where in the second last inequality, we upper bound $\delta_{\ell}$ using $v_i \leq 1$ and Assumption~\ref{ass:bounded}, and the last inequality is due to Lemma~\ref{lem:conf-bound}. By Lemma~\ref{lem:monomax} and Assumption~\ref{ass:lipschitz}, we get 
\begin{align}
 & \E[ \delta_{\ell} \cond \mathcal{E}_{\ell} ] \nonumber\\
 & \leq \E \left[ (1 + \sum_{i \in S_{\ell} } v_i) ( U( F(S_{\ell},\tilde{\bm{v}}^{\ell} ) - U( F(S_{\ell},\bm{v} )) \cond \mathcal{E}_{\ell} \right] \nonumber \\
 & \leq \E \left[ \gamma_2 \sum_{i \in S_{\ell} } \left( C_1  \sqrt{\frac{ v_i \ln( \sqrt{N} \ell + 1 ) }{T_i(\ell) \vee 1} } \right. \right. \nonumber \\
& \quad \left. \left. + C_2 \frac{ \ln( \sqrt{N} \ell + 1 ) }{T_i(\ell) \vee 1}  \right) \right], \label{main-thm:equ-21}
\end{align}
where in the last equality we have used the definition of event $\mathcal{E}_{\ell}$. By \eqref{main-thm:equ-2} and \eqref{main-thm:equ-21}, we have
\begin{align} 
& \E \delta_{\ell} \leq \frac{2\gamma_1(N+1) }{ \ell } +  \E \left[ \gamma_2 \sum_{i \in S_{\ell} } \left( C_1  \sqrt{\frac{ v_i \ln( \sqrt{N} \ell + 1 ) }{T_i(\ell) \vee 1} } \right. \right. \nonumber \\
& \qquad \left. \left. + C_2 \frac{ \ln( \sqrt{N} \ell + 1 ) }{T_i(\ell) \vee  1} \right) \right]. \label{main-thm:equ-3}
\end{align}
Putting \eqref{main-thm:equ-3} back into \eqref{main-thm:equ-1}, we derive
\begin{align}
    & \reg_T \nonumber \\
    & \leq 2\gamma_1(N + 1) \underbrace{ \E \left[ \sum_{\ell = 1}^L \frac{1}{ \ell } \right] }_{(*)} \nonumber\\
    & \quad + \gamma_2 C_1  \sqrt{ \ln(\sqrt{N}T + 1) } \underbrace{ \E \left[  \sum_{\ell = 1}^L \sum_{i \in S_{\ell} } \sqrt{\frac{ v_i }{T_i(\ell) \vee 1} }  \right]}_{(**)} \nonumber \\
    & \quad +\gamma_2 C_2 \ln(\sqrt{N}T + 1) \underbrace{ \E \left[  \sum_{\ell = 1}^L \sum_{i \in S_{\ell} } \frac{1}{T_i(\ell) \vee 1}  \right]}_{(***)}. \label{main-thm:equ-4}
\end{align}
Note that $(*) \leq \sum_{\ell = 1}^{T} {\ell}^{-1} \leq \ln T + \gamma $, where $\gamma$ is Euler's constant. Next we bound $(**)$. Let $\mathcal{K}$ be the set $\{ (i, \ell) : T_i(\ell) = 0 \}$. It is easy to see $|\mathcal{K}| \leq N$. Together with observation $\sum_{i=1}^{j} \frac{1}{\sqrt{i}} \leq 2\sqrt{j}$ and Jensen's inequality, we derive
\begin{align*}
(**) &  \leq N + \E \left[ \sum_{i \in [N]} \left(\sqrt{v_i}  \sum_{j = 1}^{n_i} \frac{1}{\sqrt{j}} \right)\right] \\
& \leq N + 2 \E\left[ \sum_{i \in [N]} \sqrt{v_i n_i} \right] \leq N + 2 \sqrt{ N \E \left[ \sum_{i \in [N]} v_i n_i \right]},
\end{align*}
where $n_i$ is the total number of episodes product $i$ is served before the $L$th episode.
Noting that
\begin{align*}
    T & \geq \E \left[ \sum_{\ell = 1}^{L-1} \E[ l_{\ell} | \mathcal{H}_{\ell} ] \right] \\
      & = \E\left[ \sum_{\ell = 1}^{L-1}(1 + \sum_{i \in S_{\ell}} v_i) \right]
      \geq \E \left[ \sum_{i \in [N] } v_i n_i \right],
\end{align*} we obtain
$
(**) \leq N + 2 \sqrt{NT}.
$
Finally, we bound $(***)$ using
\begin{align*}
(***) & \leq N + \E\left[ \sum_{i \in [N]} ( \ln n_i + \gamma) \right] \\
& \leq N (1+\gamma) + N \ln \frac{\sum_{i \in [N]} n_i}{N} \\ 
& \leq N (1+\gamma) + N \ln T
\end{align*} Putting inequalities of $(*)$, $(**)$ and $(***)$ back into \eqref{main-thm:equ-4}, we obtain
\begin{align*}
\reg_T \leq{}  &  2 \gamma_1 (N +1) (\ln T + \gamma) \\
& + \gamma_2 C_1 \sqrt{\ln(\sqrt{N}T + 1)}( 2\sqrt{NT}+N) \\
& + \gamma_2 C_2 \ln(\sqrt{N}T+1)( N \ln T + N(1+\gamma)) \\
={} & \tilde{\mathcal{O}}( \sqrt{NT})
\end{align*}
and the proof is complete. 
\end{proof}

\section{Examples of Risk Criteria}
In this section, we show that conditional value-at-risk, Sharpe Ratio, and entropy risk all satisfy Assumption~\ref{ass:bounded} and Assumption~\ref{ass:lipschitz}. For the proof of the other risk criteria listed in Table~\ref{table-of-reward-function}, we refer to Appendix A. 

For proving the one-sided Lipschitz condition, the following lemma is useful. The proof of Lemma \ref{LipschitzLemma} is in Appendix A. 
\begin{lemma}\label{LipschitzLemma}
For any $\bm{v}' \geq \bm{v}$, i.e., $v_i'\geq v_i$ for all $i\in[N]$, and $S \subset [N]$, it holds that
\begin{multline}
    \sum_{i\in S}\left\vert \frac{v'_i}{1+\sum_{i\in S}v'_i}- \frac{v_i}{1+\sum_{i\in S}v_i}\right\vert \leq \\  \frac{2}{1+\sum_{i\in S}v_i}\left[\sum_{i \in S} (v_i' - v_i) \right].
\end{multline}
\end{lemma}

\subsection{Conditional Value-at-risk}
Given $\alpha \in (0, 1]$, the conditional value-at-risk at $\alpha$ percentile for $F \in \mathcal{D}[0, 1]$ is defined as
$$
\cvar_{\alpha}(F) \eqdef \frac{1}{\alpha} \int_0^{
\alpha} \var_{\beta}(F) \diff \beta.
$$
An equivalent definition is
\begin{align*}
\cvar _\alpha(F) =\frac{1}{\alpha} \left(\alpha -\int_0^1 (F(x)\wedge \alpha) \diff x \right).
\end{align*}

\begin{proposition} \label{lem:cvar-ass2}
$\cvar_{\alpha}$ satisfies Assumption~\ref{ass:bounded} and Assumption~\ref{ass:lipschitz} with $\gamma_1 = 1$ and $\gamma_2 = 3 /\alpha$.
\end{proposition}

\begin{proof}
It is easy to see that $|\cvar_{\alpha}(F(S, \bm{v}))| \leq 1$, which implies $\gamma_1 = 1$. 

We now show the value of $\gamma_2$. Without loss of generality, we can assume that the profit of different products are different since for those items with the same revenue, we can combine them into one product and the corresponding abstraction parameter is the sum of those of the arms. Given assortment $S$, we denote the products in $S$ by $[\vert S \vert]$ in the increasing order of their profit. Then for any $k+1\in [\vert S \vert]$ and $x\in[r_k,r_{k+1})$
\begin{align*}
F(S, \bm{v}'; x) - F(S, \bm{v}; x) = \frac{1+\sum_{i=1}^k v_i' }{ 1+\sum_{i=1}^{|S|} v_i'}-\frac{1+\sum_{i=1}^k v_i}{1+\sum_{i=1}^{|S|} v_i }.
\end{align*}
Note that for any $x \in [r_k, r_{k+1})$, by Lemma \ref{LipschitzLemma}
\begin{align*}
&\vert F(S, \bm{v}'; x) - F(S,\bm{v};x)\vert \\
&\leq \left\vert \frac{1}{ 1+\sum_{i=1}^{|S|} v_i'}-\frac{1}{1+\sum_{i=1}^{|S|} v_i } \right\vert\\
&\quad + \sum_{i=1}^k \left\vert \frac{v_i'}{ 1+\sum_{i=1}^{|S|} v_i'}-\frac{v_i}{1+\sum_{i=1}^{|S|} v_i } \right\vert\\
&\leq \frac{\sum_{i=1}^{|S|}\vert v_i'-v_i\vert }{\left(1+\sum_{i=1}^{|S|} v_i' \right)\left(1+\sum_{i=1}^{|S|} v_i \right)}\\
&\quad + \frac{2}{1+\sum_{i=1}^{|S|}v_i}\left[\sum_{i \in S} (v_i' - v_i) \right]\\
&\leq \frac{3}{1+\sum_{i=1}^{|S|}v_i}\left[\sum_{i \in S} (v_i' - v_i) \right].
\end{align*}
Clearly the difference between $F(S, \bm{v}'; x) \wedge \alpha $ and $F(S, \bm{v}; x)\wedge \alpha $ satisfies the same bound. Hence
\begin{align*}
 &\left\vert\cvar_{\alpha}( F(S, \bm{v}') )-\cvar_{\alpha}( F(S, \bm{v}) ) \right\vert \\
 &\leq \frac{1}{\alpha}\left( \int_0^1 \vert F_X(x)\wedge \alpha -F_Y(x) \wedge \alpha \vert \diff x \right)\\
& \leq \frac{3/\alpha}{1+\sum_{i=1}^{|S|} v_i} \left[\sum_{i=1}^{|S|} (v_i' - v_i) \right].
\end{align*}
\end{proof}

\begin{figure*}[t]
\centering
\begin{minipage}[b]{.5\textwidth}
\centering
\includegraphics[width=.94\textwidth]{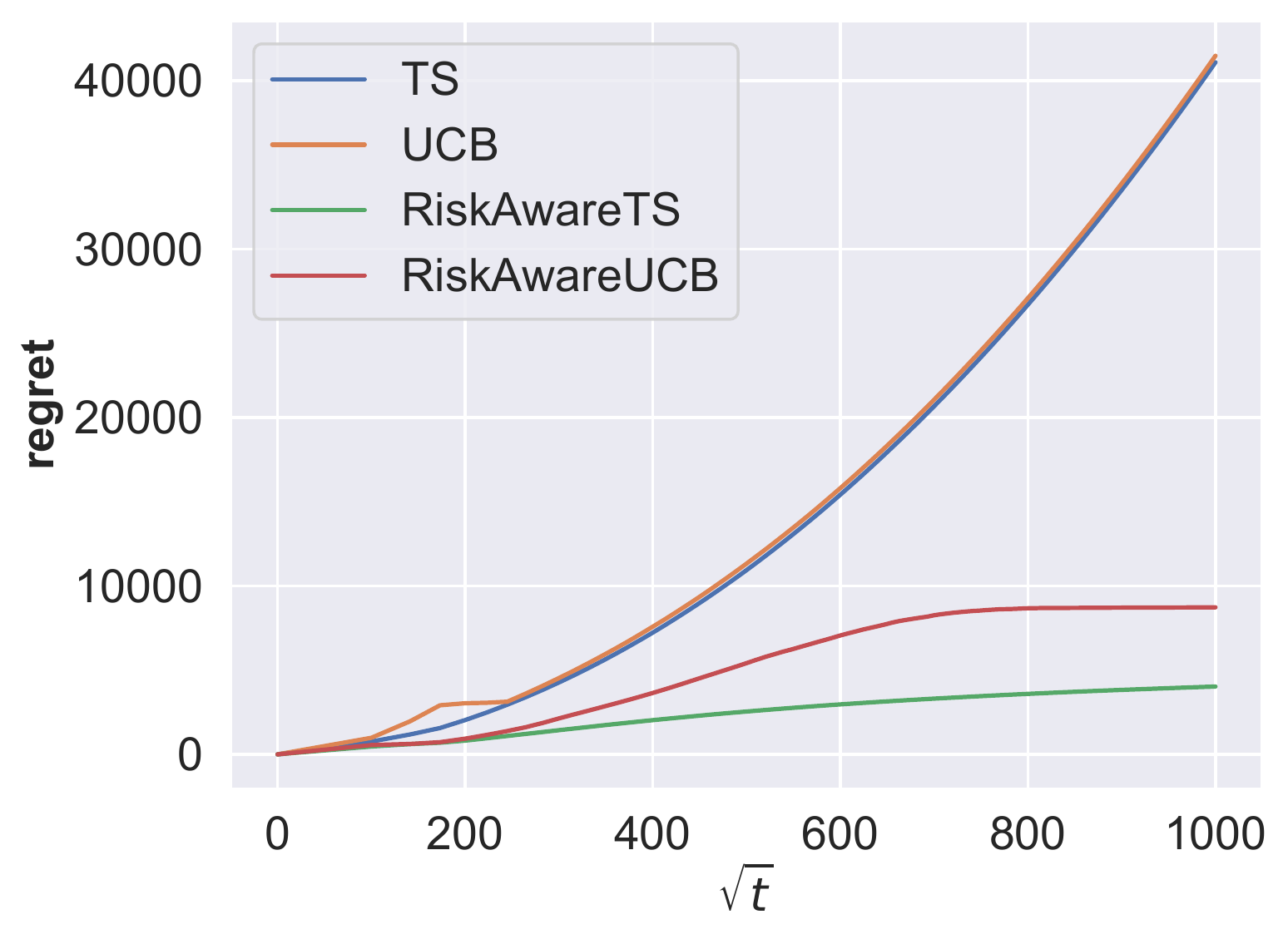}
\caption{Synthetic Data} \label{fig:worst-regret}
\end{minipage}%
\begin{minipage}[b]{.5\textwidth}
\centering
\includegraphics[width=.9\textwidth]{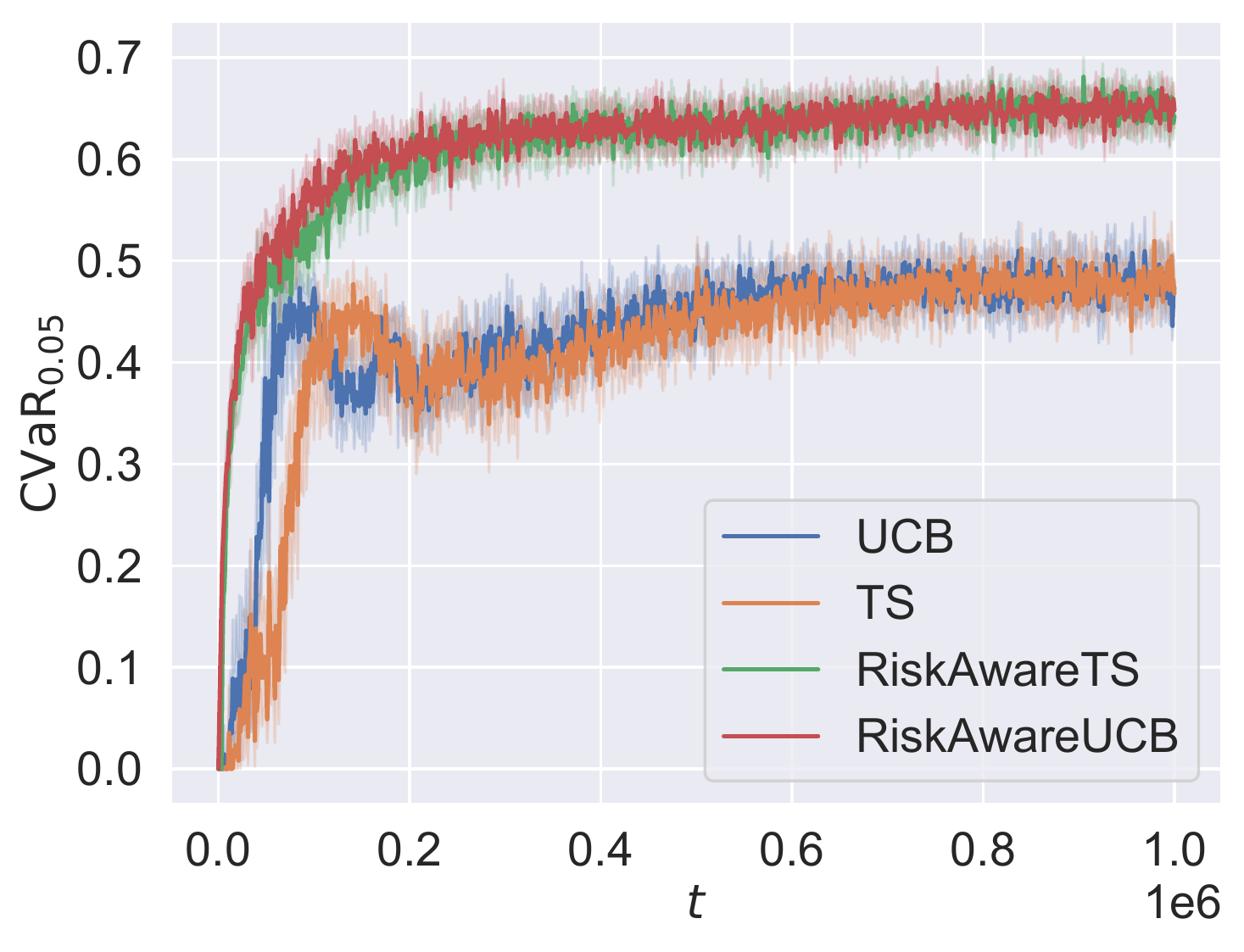}
\caption{Real Data}\label{fig:cvar}
\end{minipage}
\end{figure*}

\subsection{Sharpe Ratio}
Given a minimum average reward $r\in[0,1]$ and the regularization factor $\epsilon$, for $F\in \mathcal{D}[0,1]$ we define 
 \begin{equation*}
     Sh_{r,\epsilon}(F)=\frac{U^1(F)-r}{\sqrt{\epsilon+\sigma^2(F)}},
 \end{equation*}
where $U^1(F)$ is the mean and $\sigma^2(F)$ is the variance. 
\begin{proposition} \label{lem:SR-ass2}
$Sh_{r,\epsilon}$ satisfies Assumption~\ref{ass:bounded} and Assumption~\ref{ass:lipschitz} with $\gamma_1 = \frac{1}{\sqrt{\epsilon}}$ and $\gamma_2 = 2\epsilon^{-1/2}+3\epsilon^{-3/2}$.
\end{proposition}
\begin{proof}
Since $\sigma^2(F)>0$, $U^1(F)\in[0,1]$ and $r\in[0,1]$, it is easy to see that $\gamma_1=\frac{1}{\sqrt{\epsilon}}$. For the value of $\gamma_2$, by the Lipschitz property of the mean and variance (see Appendix A), we have 
\begin{align*}
    &\vert Sh_{r,\epsilon}(F(S,\bm{v}'))-Sh_{r,\epsilon}(F(S,\bm{v}))\vert\\
    &\leq \left\vert \frac{U^1(F(S,\bm{v}'))-U^1(F(S,\bm{v}))}{\sqrt{\epsilon+\sigma^2(F(S,\bm{v}'))}}\right\vert +\vert U^1(F(S,\bm{v}))-r\vert \\
    & \quad \times \left\vert \frac{1}{\sqrt{\epsilon+\sigma^2(F(S,\bm{v}'))}} -\frac{1}{\sqrt{\epsilon+\sigma^2(F(S,\bm{v}))}}\right\vert \\
    & \leq \frac{1}{\sqrt{\epsilon}}\vert U^1(F(S,\bm{v}'))-U^1(F(S,\bm{v}))\vert \\
    & \quad +\frac{\vert \sigma^2(F(S,\bm{v}'))-\sigma^2(F(S,\bm{v}))\vert }{\sqrt{\epsilon+\sigma^2(F(S,\bm{v}'))}\sqrt{\epsilon+\sigma^2(F(S,\bm{v}))} } \\
    & \quad \times \frac{1}{\sqrt{\epsilon+\sigma^2(F(S,\bm{v}'))}+\sqrt{\epsilon+\sigma^2(F(S,\bm{v}))}} \\
    &\leq \frac{1}{\sqrt{\epsilon}}\vert U^1(F(S,\bm{v}'))-U^1(F(S,\bm{v}))\vert \\
    & \quad +\frac{1}{2\epsilon^{\frac{3}{2}}}\vert \sigma^2(F(S,\bm{v}'))-\sigma^2(F(S,\bm{v}))\vert\\
    &\leq \frac{2\epsilon^{-1/2}+3\epsilon^{-3/2}}{1+\sum_{i\in S}v_i}\left[\sum_{i \in S} (v_i' - v_i) \right].
\end{align*}
\end{proof}

\subsection{Entropy Risk}
Given the risk aversion parameter $\theta>0$, the entropy risk measure for $F \in \mathcal{D}[0, 1]$ is defined as 
\begin{equation*}
    U^{ent}(F)=-\frac{1}{\theta}\ln \left( \int_{0}^1 e^{-\theta x} \diff F(x) \right).
\end{equation*}

\begin{proposition} \label{lem:ent-ass2}
$U^{ent}$ satisfies Assumption~\ref{ass:bounded} and Assumption~\ref{ass:lipschitz} with $\gamma_1 = 1$ and $\gamma_2 = 2e^\theta /\theta$.
\end{proposition}
\begin{proof}
By Jensen's inequality, we always have $\vert U^{ent}(F)\vert\leq 1$. Given a fixed $\theta>0$, we know that 
\begin{equation*}
    \sum_{i\in S}e^{-\theta r_i} \frac{v_i}{1+\sum_{i\in S} v_i} \in [e^{-\theta},1].
\end{equation*}
By the convexity of the log function and Lemma \ref{LipschitzLemma}, we have that 
\begin{align*}
    &\vert U^{ent}(F(S,\bm{v}')) -U^{ent}(F(S,\bm{v})) \vert\\
    &= \frac{1}{\theta}\left\vert \ln \left(\sum_{i\in S} \frac{e^{-\theta r_i} v'_i}{1+\sum_{i\in S} v'_i}\right)-\ln\left(\sum_{i\in S} \frac{e^{-\theta r_i} v_i}{1+\sum_{i\in S} v_i}\right) \right\vert \\
    &\leq \frac{e^\theta}{\theta} \left\vert \sum_{i\in S} \frac{e^{-\theta r_i} v'_i}{1+\sum_{i\in S} v'_i}- \sum_{i\in S} \frac{e^{-\theta r_i} v_i}{1+\sum_{i\in S} v_i}\right\vert \\
    &\leq \frac{2e^\theta /\theta}{1+\sum_{i\in S}v_i}\left[\sum_{i \in S} (v_i' - v_i) \right].
\end{align*}
\end{proof}

\section{Experimental Evaluation}
We evaluate $\mathtt{RiskAwareUCB}$ and $\mathtt{RiskAwareTS}$ against $\mathtt{UCB}$ and $\mathtt{TS}$ where the last two algorithms are set to maximize the expected revenue in both synthetic and real data. Based on the BanditPyLib library \cite{BanditPyLib}, all of the algorithms\footnote{Please refer to \url{https://github.com/Alanthink/aaai2021} for the source code.} are implemented in Python3.

\paragraph{Synthetic Data} In this experiment, we fix the number of products $N = 10$, cardinality limit $K = 4$, horizon $T = 10^6$, and set the goal to be $U = \cvar_{0.5}$. We generate $10$ uniformly distributed random input instances where $v_i \in [0, 1]$ and $r_i \in [0.1, 1]$. For each input instance, we run $20$ repetitions and compute their average as the regret. Figure~\ref{fig:worst-regret} shows how the worst regret among all input instances changes with square root of time. 

\paragraph{Real Data} In this experiment, we consider the ``UCI Car Evaluation Database" dataset from the Machine Learning Repository \cite{Dua:2019} which
contains $6$ categorical attributes for $N = 1728$ cars and consumer ratings for each car. We fix cardinality limit $K = 100$, horizon $T = 10^6$, and set the goal to be $U = \cvar_{0.05}$. 

By transforming each attribute to a one-hot vector, we obtain an attribute vector $m_i \in \{0, 1\}^{21}$ for each car. There are four different values for customer ratings i.e., ``acceptable", ``good", ``very good", and ``unacceptable". We decode ``unacceptable" by $0$ and others by $1$ to represent whether the customer has the intention to buy the car. We use logistic regression to predict whether the customer is likely to buy the car and the probability that the customer buys car $i$ is modeled by
$$
\frac{1}{1 + \exp( - \theta^{\mathrm{T}} m_i) },
$$ where $\theta \in \mathbb{R}^{21}$ is an \textit{unknown} parameter. After the model is fit with $L_2$ regularization, we set the preference parameter $v_i$ of car $i$ to be the same as the probability predicted by logistic regression. Since there is no profit data available for cars in this dataset, we generate uniformly distributed profit $r_i$ from $[0.1, 1]$ for each car. 

We run the experiment for $40$ repetitions and compute the average $\cvar_{0.05}$ for every consecutive $1000$ revealed profits. To save time, when computing the assortment with the best $\cvar_{0.05}$, we do a local search, i.e., try to replace a car, add a car or delete a car, and stops if we can not find a better assortment. Figure~\ref{fig:cvar} shows the results of the experiment.

\paragraph{Discussion} From Figure~\ref{fig:worst-regret}, we can see that both $\mathtt{RiskAwareUCB}$ and $\mathtt{RiskAwareTS}$ suffer a $\sqrt{t}$-rate regret. Moreover, $\mathtt{RiskAwareTS}$ performs better than $\mathtt{RiskAwareUCB}$, which aligns with literature that Thompson Sampling performs better in practice. From Figure~\ref{fig:cvar}, we can see that the obtained $\cvar_{0.05}$ under $\mathtt{UCB}$ and $\mathtt{TS}$ are far from optimal. However, $\mathtt{RiskAwareUCB}$ and $\mathtt{RiskAwareTS}$ perform roughly the same. For both of these experiments, we can see the proposed algorithms $\mathtt{RiskAwareUCB}$ and $\mathtt{RiskAwareTS}$ perform much better than $\mathtt{UCB}$ and $\mathtt{TS}$.

\section{Conclusion}
In this work, we have shown the near-optimal algorithms for a general class of risk criteria, which only need to satisfy three mild assumptions. Experiments with both synthetic and real data are conducted to validate our results and show that ordinary algorithms suffer a worse performance when the goal is changed.

\section{Acknowledgments}
This work was supported in part by NSF IIS-1633215, CCF-1844234, CCF-2006591, and CCF-2006526.

\bibliography{references}

\begin{thebibliography}{15}
\providecommand{\natexlab}[1]{#1}
\providecommand{\url}[1]{\texttt{#1}}
\providecommand{\urlprefix}{URL }
\expandafter\ifx\csname urlstyle\endcsname\relax
  \providecommand{\doi}[1]{doi:\discretionary{}{}{}#1}\else
  \providecommand{\doi}{doi:\discretionary{}{}{}\begingroup
  \urlstyle{rm}\Url}\fi

\bibitem[{Agrawal et~al.(2017)Agrawal, Avadhanula, Goyal, and
  Zeevi}]{DBLP:conf/colt/AgrawalAGZ17}
Agrawal, S.; Avadhanula, V.; Goyal, V.; and Zeevi, A. 2017.
\newblock Thompson Sampling for the MNL-Bandit.
\newblock In \emph{Conference on Learning Theory (COLT)}.

\bibitem[{Agrawal et~al.(2019)Agrawal, Avadhanula, Goyal, and
  Zeevi}]{agrawal2019mnl}
Agrawal, S.; Avadhanula, V.; Goyal, V.; and Zeevi, A. 2019.
\newblock MNL-bandit: A dynamic learning approach to assortment selection.
\newblock \emph{Operations Research} 67(5): 1453--1485.

\bibitem[{Cassel, Mannor, and Zeevi(2018)}]{cassel2018general}
Cassel, A.; Mannor, S.; and Zeevi, A. 2018.
\newblock A general approach to multi-armed bandits under risk criteria.
\newblock In \emph{Conference on Learning Theory (COLT)}.

\bibitem[{Chen and Wang(2018)}]{chen2018note}
Chen, X.; and Wang, Y. 2018.
\newblock A note on a tight lower bound for capacitated MNL-bandit assortment
  selection models.
\newblock \emph{Operations Research Letters} 46(5): 534--537.

\bibitem[{Dong et~al.(2020)Dong, Li, Zhang, and
  Zhou}]{DBLP:journals/corr/abs-2007-04876}
Dong, K.; Li, Y.; Zhang, Q.; and Zhou, Y. 2020.
\newblock Multinomial Logit Bandit with Low Switching Cost.
\newblock In \emph{International Conference on Machine Learning (ICML)}.

\bibitem[{Dua and Graff(2017)}]{Dua:2019}
Dua, D.; and Graff, C. 2017.
\newblock {UCI} Machine Learning Repository.
\newblock \url{http://archive.ics.uci.edu/ml} (Accessed on 2020-09-08).

\bibitem[{Galichet, Sebag, and Teytaud(2014)}]{DBLP:journals/corr/GalichetST14}
Galichet, N.; Sebag, M.; and Teytaud, O. 2014.
\newblock Exploration vs Exploitation vs Safety: Risk-averse Multi-Armed
  Bandits.
\newblock In \emph{Asian Conference on Machine Learning (ACML)}.

\bibitem[{Holtz, Tao, and Xi(2020)}]{BanditPyLib}
Holtz, C.; Tao, C.; and Xi, G. 2020.
\newblock {BanditPyLib: a lightweight python library for bandit algorithms}.
\newblock Online at: \url{https://github.com/Alanthink/banditpylib}.
\newblock \urlprefix\url{https://github.com/Alanthink/banditpylib}.
\newblock Documentation at \url{https://alanthink.github.io/banditpylib-doc}.

\bibitem[{Kamath(2015)}]{Kamath20}
Kamath, G. 2015.
\newblock Bounds on the Expectation of the Maximum of Samples from a Gaussian.
\newblock \url{http://www.gautamkamath.com/writings/gaussian_max.pdf} (Accessed
  on 2020-09-08).

\bibitem[{Maillard(2013)}]{DBLP:conf/alt/Maillard13}
Maillard, O. 2013.
\newblock Robust Risk-Averse Stochastic Multi-armed Bandits.
\newblock In \emph{Algorithmic Learning Theory (ALT)}.

\bibitem[{Rusmevichientong, Shen, and
  Shmoys(2010)}]{DBLP:journals/ior/RusmevichientongSS10}
Rusmevichientong, P.; Shen, Z.~M.; and Shmoys, D.~B. 2010.
\newblock Dynamic Assortment Optimization with a Multinomial Logit Choice Model
  and Capacity Constraint.
\newblock \emph{Operations Research} 58(6): 1666--1680.

\bibitem[{Sani, Lazaric, and Munos(2012)}]{DBLP:conf/nips/SaniLM12}
Sani, A.; Lazaric, A.; and Munos, R. 2012.
\newblock Risk-Aversion in Multi-armed Bandits.
\newblock In \emph{Advances in Neural Information Processing Systems (NIPS)},
  3284--3292.

\bibitem[{Saur{\'{e}} and Zeevi(2013)}]{DBLP:journals/msom/SaureZ13}
Saur{\'{e}}, D.; and Zeevi, A. 2013.
\newblock Optimal Dynamic Assortment Planning with Demand Learning.
\newblock \emph{Manufacturing {\&} Service Operations Management} 15(3):
  387--404.

\bibitem[{Vakili and Zhao(2016)}]{DBLP:journals/jstsp/VakiliZ16}
Vakili, S.; and Zhao, Q. 2016.
\newblock Risk-Averse Multi-Armed Bandit Problems Under Mean-Variance Measure.
\newblock \emph{IEEE Journal of Selected Topics in Signal Processing} 10(6):
  1093--1111.

\bibitem[{Zimin, Ibsen{-}Jensen, and
  Chatterjee(2014)}]{DBLP:journals/corr/ZiminIC14}
Zimin, A.; Ibsen{-}Jensen, R.; and Chatterjee, K. 2014.
\newblock Generalized Risk-Aversion in Stochastic Multi-Armed Bandits.
\newblock \emph{arXiv preprint arXiv:1405.0833} .

\end{thebibliography}

\newpage
\onecolumn

\section{Appendix A: Lipschitz Condition}
In this section, we prove that the $n$th-moments, below target semi-variance, negative variance, mean-variance and Sortino ratio all satisfy Assumption \ref{ass:bounded} and \ref{ass:lipschitz}. Firstly, we give the proof to Lemma \ref{LipschitzLemma} which is useful for verifying Assumption \ref{ass:lipschitz}.
\begin{proof}[Proof of Lemma~\ref{LipschitzLemma}]
Given $\bm{v}'\geq  \bm{v}$, define the subset 
\begin{equation*}
    S_0=\left\{i\in S: \frac{v'_i}{1+\sum_{i\in S}v'_i}\geq  \frac{v_i}{1+\sum_{i\in S}v_i}\right\}.
\end{equation*}
Then 
\begin{align*}
    \sum_{i\in S_0}\left\vert \frac{v'_i}{1+\sum_{i\in S}v'_i}- \frac{v_i}{1+\sum_{i\in S}v_i}\right\vert
    & = \sum_{i\in S/ S_0}\left\vert \frac{v'_i}{1+\sum_{i\in S}v'_i}- \frac{v_i}{1+\sum_{i\in S}v_i}\right\vert +\left\vert \frac{1}{1+\sum_{i\in S}v'_i}- \frac{1}{1+\sum_{i\in S}v_i}\right\vert.
\end{align*}
Therefore
\begin{align*}
    \sum_{i\in S}\left\vert \frac{v'_i}{1+\sum_{i\in S}v'_i}- \frac{v_i}{1+\sum_{i\in S}v_i}\right\vert & \leq 2\sum_{i\in S_0}\left\vert \frac{v'_i}{1+\sum_{i\in S}v'_i}- \frac{v_i}{1+\sum_{i\in S}v_i}\right\vert\\
    &\leq 2\sum_{i\in S_0}\left( \frac{v'_i}{1+\sum_{i\in S}v'_i}- \frac{v_i}{1+\sum_{i\in S}v'_i}\right)\\
     &\leq \frac{2}{1+\sum_{i\in S}v'_i}\left[\sum_{i \in S} (v_i' - v_i) \right]\\
    &\leq \frac{2}{1+\sum_{i\in S}v_i}\left[\sum_{i \in S} (v_i' - v_i) \right].
\end{align*}
\end{proof}
\begin{remark}\label{remark:strongResult}
For any $\bm{v}'\geq \bm{v}$, from the proof above, we can see that the following stronger result is true
\begin{equation}
    \sum_{i\in S}\left\vert \frac{v'_i}{1+\sum_{i\in S}v'_i}- \frac{v_i}{1+\sum_{i\in S}v_i}\right\vert \leq \frac{2}{1+\sum_{i\in S}v'_i}\left[\sum_{i \in S} (v_i' - v_i) \right].
\end{equation}

\end{remark}
\subsection{Value-at-risk}

Given $\alpha \in (0, 1]$, value-at-risk at $\alpha$ percentile for $F \in \mathcal{D}[0, 1]$ is defined as
$$
\var_{\alpha}(F) \eqdef \inf\{ x : F(x) \geq \alpha \}.
$$
It is easy to see that $|\var_{\alpha}(F(S, \bm{v}))| \leq 1$ and hence $\gamma_1 = 1$. However, $\var_{\alpha}(F(S, \bm{v}))$ is not continuous on $\bm{v}$, and $\gamma_2$ does not exist.

\subsection{$n$th-moment}
For any $n\in \mathbb{N}$, we can define the $n$th-moment about zero for $F \in \mathcal{D}[0, 1]$
\begin{equation*}
    U^n(F)=\int_{0}^1 x^n \diff F(x).
\end{equation*}
\begin{proposition} \label{lem:nMoment-ass2}
The $n$th-moment satisfies Assumption~\ref{ass:bounded} and Assumption~\ref{ass:lipschitz} with $\gamma_1 = 1$ and $\gamma_2 = 1$.
\end{proposition}
\begin{proof}
It is immediate to check that $\vert U^n(F)\vert\leq 1$. For the value of $\gamma_2$, notice that 
\begin{align*}
     U^n (F(S,\bm{v}'))  - U^n( F(S,\bm{v})) 
     & = \sum_{i\in S} r_i^n\left[ \frac{v'_i}{1+\sum_{i\in S} v'_i}-\frac{v_i}{1+\sum_{i\in S} v_i} \right]\\
    & \leq \sum_{i\in S} r_i^n\left[ \frac{v'_i}{1+\sum_{i\in S} v_i}-\frac{v_i}{1+\sum_{i\in S} v_i} \right]\\
    &\leq \frac{ 1 }{1+\sum_{i \in S} v_i} \left[\sum_{i \in S} (v_i' - v_i) \right].
\end{align*}
Moreover, using Lemma \ref{LipschitzLemma}, we have that 
\begin{equation} \label{eq:nmomentLipschitz}
    \vert U^n (F(S,\bm{v}'))  - U^n( F(S,\bm{v}))\vert  \leq \frac{ 2 }{1+\sum_{i \in S} v_i} \left[\sum_{i \in S} (v_i' - v_i) \right].
\end{equation}
\end{proof}
\begin{remark}
It is easy to see that the same argument applies to a large class of risk criteria of the form 
\begin{equation*}
    U(F)=\int_{-\infty}^\infty h(x) \diff F(x).
\end{equation*}
\end{remark}

\subsection{Below target semi-variance}
Given a target $r\in [0,1]$, we can define the negative below target semi-variance for any $F \in \mathcal{D}[0, 1]$
\begin{equation*}
    -TSV_r(F)= - \int_0^1 (x-r)^2 \ind{\{ x\leq r\}} \diff F(x).
\end{equation*}
\begin{proposition} \label{lem:BTV-ass2}
The negative below target semi-variance satisfies Assumption~\ref{ass:bounded} and Assumption~\ref{ass:lipschitz} with $\gamma_1 = r^2$ and $\gamma_2 = 2r^2$.
\end{proposition}
\begin{proof}
By the definition of $-TSV_R(F)$, it is easy to see that it is bounded by $r^2$. Then by Lemma \ref{LipschitzLemma}, there is 
\begin{align*}
    \vert -TSV_r(F(S,\bm{v}')) +TSV_r(F(S,\bm{v})) \vert 
    & = \left\vert \sum_{i\in S,r_i\leq r} (r_i -r)^2\left(-\frac{v'_i}{1+\sum_{i\in S}v'_i} +\frac{v_i}{1+\sum_{i\in S}v_i} \right)\right\vert\\
    &\leq  r^2 \sum_{i\in S}  \left\vert \frac{v'_i}{1+\sum_{i\in S}v'_i}- \frac{v_i}{1+\sum_{i\in S}v_i}\right\vert\\ &\leq \frac{2 r^2}{1+\sum_{i\in S}v_i}\left[\sum_{i \in S} (v_i' - v_i) \right].
\end{align*}
\end{proof}

\subsection{Negative variance}
For any $F\in \mathcal{D}[0,1]$, the negative variance is 
\begin{equation*}
    -\sigma^2(F)=-[U^2(F)-(U^1(F))^2].
\end{equation*}
\begin{proposition} \label{lem:Var-ass2}
The negative variance satisfies Assumption~\ref{ass:bounded} and Assumption~\ref{ass:lipschitz} with $\gamma_1 = \frac{1}{4}$ and $\gamma_2 = 6$.
\end{proposition}
\begin{proof}
It is well known that $\vert \sigma^2(F)\vert \leq \frac{1}{4}$ for any random variable taking values in $[0,1]$. For the value of $\gamma_2$, there is 
\begin{equation*}
    \vert -\sigma^2(F(S,\bm{v}'))+\sigma^2(F(S,\bm{v}))\vert \leq \vert U^2(F(S,\bm{v}'))-U^2(F(S,\bm{v}))\vert +\vert (U^1(F(S,\bm{v}')))^2-(U^1(F(S,\bm{v})))^2\vert.
\end{equation*}
By \eqref{eq:nmomentLipschitz} and 
\begin{align*}
    \vert (U^1(F(S,\bm{v}')))^2-(U^1(F(S,\bm{v})))^2\vert & \leq \vert U^1(F(S,\bm{v}'))-U^1(F(S,\bm{v}))\vert \vert U^1(F(S,\bm{v}'))+U^1(F(S,\bm{v}))\vert\\
    &\leq 2\vert U^1(F(S,\bm{v}'))-U^1(F(S,\bm{v}))\vert,
\end{align*}
 we have that 
 \begin{equation*}
     \vert -\sigma^2(F(S,\bm{v}'))+\sigma^2(F(S,\bm{v}))\vert \leq\frac{6}{1+\sum_{i\in S}v_i}\left[\sum_{i \in S} (v_i' - v_i) \right].
 \end{equation*}
\end{proof}

\begin{remark}
Simply consider the negative variance alone does not provide a good risk criterion for our problem here. But we still discuss it here because it is an important building block for other risk criteria including the Sharpe ratio, Sortino ratio, and mean-variance. 
\end{remark}

\subsection{Mean-variance}
Given a weight $\rho>0$, we define the the mean-variance for $F\in \mathcal{D}[0,1]$ as 
\begin{equation*}
    U^{MV}_\rho(F)=U^1(F)-\rho \sigma^2(F).
\end{equation*}
By the boundedness and the Lipschitz property of the mean and variance, it is immediate to see that $\vert U^{MV}_\rho(F) \vert\leq 1+\frac{\rho}{4}$ and  
\begin{equation*}
    \vert U^{MV}_\rho(F(S,\bm{v}'))-U^{MV}_\rho(F(S,\bm{v}))\vert \leq\frac{2+6\rho}{1+\sum_{i\in S}v_i}\left[\sum_{i \in S} (v_i' - v_i) \right].
\end{equation*}

\subsection{Sortino ratio}
 Given a minimum average reward $r\in[0,1]$ and the regularization factor $\epsilon$, for $F\in \mathcal{D}[0,1]$ we define 
 \begin{equation*}
     So_{r,\epsilon}(F)=\frac{U^1(F)-r}{\sqrt{\epsilon+TSV_r(F)}}.
 \end{equation*}
\begin{proposition} \label{lem:Sortino-ass2}
$So_{r,\epsilon}(F)$ satisfies Assumption~\ref{ass:bounded} and Assumption~\ref{ass:lipschitz} with $\gamma_1 = \frac{1}{\sqrt{\epsilon}}$ and $\gamma_2 = 2\epsilon^{-1/2}+\epsilon^{-3/2}$.
\end{proposition}
\begin{proof}
Following similar argument for the Sharpe ratio, using the boundedness and the Lipschitz property of the mean and the below target semi-variance, we have that $So_{r,\epsilon}(F)\leq \frac{1}{\sqrt{\epsilon}}$ and 

\begin{equation*}
\vert So_{r,\epsilon}(F(S,\bm{v}'))-So_{r,\epsilon}(F(S,\bm{v}))\vert \leq \frac{2\epsilon^{-1/2}+\epsilon^{-3/2}}{1+\sum_{i\in S}v_i}\left[\sum_{i \in S} (v_i' - v_i) \right].
\end{equation*}
\end{proof}

\section{Appendix B: Thompson Sampling} \label{app:thompson-sampling}
We prove the guarantee of $\mathtt{RiskAwareTS}$ in this section. 

Before proving the theoretical guarantee, we first present the high-level idea of the $\mathtt{RiskAwareTS}$ algorithm. At the beginning of the algorithm, there is a warm start stage when every product is repeatedly served until a no-purchase outcome is observed. Similar to $\mathtt{RiskAwareUCB}$, the remaining time steps are divided into small episodes. During each episode $\ell$, the same assortment $S_{\ell}$ is repeatedly provided to the online user until a no-purchase outcome is observed. Specifically, in each episode $\ell$, we are providing the assortment
\begin{equation*}
     \argmax_{ S \subset [N], |S| \leq K } U (F(S,\tilde{\bm{v}}^{\ell})),
\end{equation*}
where $\tilde{\bm{v}}^{\ell}$ is the virtual preference parameters generated by correlated sampling before the start of episode $\ell$.

Let $t_{i, \ell}$ be the number of times the online users buy product $i$ in the $\ell$th episode and $\mathcal{T}_i(\ell)$ be the collection of episodes for which product $i$ is served until episode $\ell$ (exclusive). Define $T_i(\ell) \eqdef |\mathcal{T}_i(\ell)|$, $n_i(\ell) \eqdef \sum_{\ell' \in \mathcal{T}_i(\ell)} t_{i, \ell' } $ and 
\begin{equation}
    \bar{v}_i^\ell \eqdef \frac{ n_i(\ell) }{ T_i(\ell) },
\end{equation}
which is an unbiased estimator of the unknown preference parameter $v_i$.
Generate $K$ \textit{i.i.d.}\ samples $\{\theta^{(j)}_{\ell} \}_{j = 1}^K$ from $\mathcal{N}(0, 1)$ and define 
$$ \mu_i^{(j)}(\ell) \eqdef \bar{v}_i^\ell + \theta_{\ell}^{(j)} \hat{\sigma}_i^\ell, $$
where 
\begin{equation*}
    \hat{\sigma}_i^\ell = \sqrt{ \frac{50 \bar{v}_i^{\ell}(\bar{v}_i^{\ell} + 1) }{T_i(\ell)} }  + \frac{75 \sqrt{\log(TK)} }{ T_i(\ell) }.
\end{equation*}
The $i$th component of the virtual preference parameters $\tilde{\bm{v}}^{\ell}$ is given by
$
\tilde{v}^{\ell}_i = \max_{1 \leq j \leq K} \mu_i^{(j)}(\ell).
$
The concrete details of $\mathtt{RiskAwareTS}$ is presented in Algorithm~\ref{alg:riskaware-ts}.

\begin{algorithm}[t]
\DontPrintSemicolon
\caption{$\mathtt{RiskAwareTS}(N, K, \bm{r}, U, T)$}
\label{alg:riskaware-ts}
Initialize $t \leftarrow 1$, $\ell \leftarrow 1$, $\mathcal{T}_i(\ell) \leftarrow \emptyset$ for $i \in [N]$ and $n_i(\ell) \leftarrow 0$ for $i \in [N]$ \\
\While {$t \leq T$}{
\If{there exists a product $i$ such that $\mathcal{T}_i(\ell) = \emptyset$}{
Pick a product $i$ such that $\mathcal{T}_i(\ell) = \emptyset$ and set $S_{\ell} \leftarrow \{i\}$
}
\Else{
\For{$j \in [K]$}{
Obtain a sample $\theta_{\ell}^{(j)}$ from Gaussian distribution $\mathcal{N}(0, 1)$
}
Let $n_i(\ell) \leftarrow \sum_{\ell' \in \mathcal{T}_i(\ell)} t_{i, \ell' }$ and $T_i(\ell) \leftarrow |\mathcal{T}_i(\ell)|$ for $i \in [N]$ \\
Set $\mu_i^{(j)}(\ell) \leftarrow \bar{v}_i^\ell + \theta_{\ell}^{(j)} \hat{\sigma}_i^\ell $ where $\bar{v}_i^\ell = \frac{n_i(\ell)}{ T_i(\ell) }$ and $\hat{\sigma}_i^\ell = \sqrt{ \frac{50 \bar{v}_i^{\ell}(\bar{v}_i^{\ell} + 1) }{T_i(\ell)} }  + \frac{75 \sqrt{\log(TK)} }{ T_i(\ell) }$ \\
$\tilde{v}^{\ell}_i \leftarrow \max_{1 \leq j \leq K} \mu_i^{(j)}(\ell)$ for $i \in [N]$ \\
$S_{\ell} \leftarrow \argmax_{ S \subset [N], |S| \leq K } U (F(S,\tilde{\bm{v}}^{\ell}))$
}
Initialize $t_{i, \ell} \leftarrow 0$ for $i \in [N]$ \\
\Repeat {$t > T$ or $c_{t-1} = 0$}{
Serve $S_{\ell}$ and observe customer choice $c_t$ \\
\lIf{$c_t \neq 0$}{$t_{c_t, \ell} \leftarrow t_{c_t, \ell} + 1$
}
$t \leftarrow t+1$
}
\For{$i \in [N]$}{
\lIf{$i \in S_{\ell}$}{$\mathcal{T}_{i}(\ell+1) \leftarrow \mathcal{T}_{i}(\ell) \cup \{ \ell \}$}
\lElse{$\mathcal{T}_{i}(\ell+1) \leftarrow \mathcal{T}_{i}(\ell)$}
}
$\ell \leftarrow \ell + 1$
}
\end{algorithm}

To prove the guarantee, we will need a stronger version of Assumption \ref{ass:lipschitz}.
\begin{assumption}[One-sided Lipschitz Condition] \label{ass:lipschitz2}
For any $\bm{v}' \geq \bm{v}$, i.e., $v_i'\geq v_i$ for all $i\in[N]$, and $S \subset [N]$, it holds that
\begin{equation*}
U (F(S,\bm{v}'))  - U( F(S,\bm{v}))  \\ 
\leq \frac{ \tilde{\gamma}_2 }{1+\sum_{i \in S} v_i'} \left[\sum_{i \in S} (v_i' - v_i) \right].
\end{equation*}
\end{assumption}

\begin{remark}
Since the verification of Assumption \ref{ass:lipschitz} for all the risk criteria above are based on Lemma \ref{LipschitzLemma}, by Remark \ref{remark:strongResult}, it is easy to prove that all the risk criteria listed above also satisfies Assumption \ref{ass:lipschitz2} following the same proof. 
\end{remark}

\begin{proposition} \label{pro:TSlipschitz}
Suppose the risk criterion $U$ satisfies Assumption \ref{ass:quasi-convex}, \ref{ass:bounded} and \ref{ass:lipschitz2}. For any $\bm{v}'$, $\bm{v}$, let $S^* \subset [N]$ be the optimal assortment for $\bm{v}$ under $K$ cardinality constraint, it holds that
\begin{equation*}
 U (F(S^*,\bm{v}))  - U( F(S^*,\bm{v}'))  \\ 
\leq \frac{ \tilde{\gamma}_2 }{1+\sum_{i \in S} \max{\{ v_i,v_i'\}}} \left[\sum_{i \in S} \vert v_i' - v_i \vert \right].
\end{equation*}
\end{proposition}
\begin{proof}
Let $\bm{u}=\max{\{\bm{v}',\bm{v}\}}$, i.e., $u_i=\max{\{v_i',v_i\}}$ for all $i\in [N]$. Since $\bm{u}\geq \bm{v}$, by Assumption \ref{ass:quasi-convex} and the proof of Lemma \ref{lem:monomax}, we know that 
\begin{equation*}
     U (F(S^*,\bm{v}))  \leq  U( F(S^*,\bm{u})). 
\end{equation*}
Therefore, using Assumption \ref{ass:lipschitz2}
\begin{align*}
     U (F(S^*,\bm{v}))  - U( F(S^*,\bm{v}'))  \leq  U (F(S^*,\bm{u}))  - U( F(S^*,\bm{v}')) \leq  \frac{ \tilde{\gamma}_2 }{1+\sum_{i \in S} u_i} \left[\sum_{i \in S} (u_i - v_i') \right].
\end{align*}
By the definition of $\bm{u}$, one can conclude the proof.
\end{proof}

Now we prove the regret upper bound for $\mathtt{RiskAwareTS}$. The proof below is a mild modification of the proof to \cite[Theorem 1]{DBLP:conf/colt/AgrawalAGZ17} and we include it here for the completeness of the paper. 
\begin{theorem} \label{thm:TSmain}
Suppose the risk criterion $U$ satisfies Assumption \ref{ass:quasi-convex}, \ref{ass:bounded} and \ref{ass:lipschitz}. The regret \eqref{eq:regret} incurred by the decision maker using $\mathtt{RiskAwareTS}$ is upper bounded by $\tilde{\mathcal{O}} ( \sqrt{NT} )$ after $T$ time steps, where $\tilde{\mathcal{O}}$ hides poly-logarithmic factors in $N$ and $T$.
\end{theorem}
\begin{proof}
For completeness, we first introduce some notations, which are already defined in \cite[Appendix D]{DBLP:conf/colt/AgrawalAGZ17}. 

Given assortment $S$, let $V(S)=\sum_{i\in S}v_i$. Given $\ell,\tau\leq L$, define 
\begin{align*}
    \Delta R_\ell &=(1+V(S_\ell))[U(F(S_\ell,\tilde{\bm{v}}^\ell ))-U(F(S_\ell,\bm{v}))],\\
    \Delta R_{\ell,\tau} &=(1+V(S_\tau))[U(F(S_\ell,\tilde{\bm{v}}^\ell ))-U(F(S_\ell,\tilde{\bm{v}}^\tau)].
\end{align*}
Next, we denote by $\mathcal{A}_0$ the probability space $\Omega$ and 
\begin{equation*}
    \mathcal{A}_\ell=\left\{ \vert \bar{v}_i(\ell)-v_i\vert\geq \sqrt{\frac{24v_i\log(\ell+1)}{T_i(\ell)}}+\frac{48\log(\ell+1)}{T_i(\ell)}, \mbox{ for some }i=1,\cdots,N \right\}.
\end{equation*}
Next we define $\mathcal{T}=\{ \ell: \tilde{v}^\ell_i \geq v_i \mbox{ for all } i\in S^*\cup S_\ell\}$, which indicates the ``optimistic" episodes. Then let $succ(\ell)=\min\{ \bar{\ell}\in \mathcal{T}:\bar{\ell}>\ell \}$, which is the next optimistic episode after episode $\ell$. Finally, we define $\mathcal{E}^{An}(\ell)=\{ \tau:\tau\in(\ell,succ(\ell)) \}$ for all $\ell\in\mathcal{T}$, which is the collection of ``non-optimistic" episodes between two adjacent optimistic episodes.

Now we consider the regret
\begin{equation}\label{eq:TSregret}
    \begin{aligned}
    \reg_T & = \E\left[ \sum_{ \ell = 1}^L l_{\ell} \left( U (F(S^*,\bm{v})) - U (F(S_{\ell},\bm{v}))\right)  \right]\\
    &=\E\left[ \sum_{ \ell = 1}^L l_{\ell} \left( U (F(S^*,\bm{v})) - U (F(S_{\ell},\tilde{\bm{v}}^\ell ))\right)  \right]+\E\left[ \sum_{ \ell = 1}^L l_{\ell} \left( U (F(S_\ell,\tilde{\bm{v}}^\ell )) - U (F(S_{\ell},\bm{v}))\right)  \right]\\
    & \eqdef \reg_T^1+\reg_T^2.
    \end{aligned}
\end{equation}
Next we will show the upper bounds of $\reg_T^1$ and $\reg_T^2$ respectively.

{\em Bounding $\reg_T^2$}: 
By taking conditional probability with respect to the history $\mathcal{H}_\ell$, following the proof of \eqref{main-thm:equ-1}, we have 
\begin{equation}\label{eq:reg2eq1}
    \reg_T^2 = \mathbb{E}\left[ \sum_{\ell=1}^L \Delta R_\ell \right].
\end{equation}
Next we bound $\Delta R_\ell$ in two scenarios
\begin{multline} \label{eq:reg2eq2}
    \mathbb{E}[\Delta R_\ell]=\mathbb{E}[\Delta R_\ell \ind{\mathcal{A}_{\ell-1}}]+\mathbb{E}[\Delta R_\ell \ind{\mathcal{A}^c_{\ell-1}}] \\ \leq 2\gamma_1(K+1)\mathbb{P}(\mathcal{A}_{\ell-1})+\mathbb{E}[\Delta R_\ell \ind{\mathcal{A}^c_{\ell-1}}]\leq \frac{2\gamma_1(K+1)}{\ell^2}+\mathbb{E}[\Delta R_\ell \ind{\mathcal{A}^c_{\ell-1}}],
\end{multline}
where the first inequality follows from Assumption \ref{ass:bounded} and $V(S)\leq K$, and the second inequality follows from \cite[Lemma 7]{DBLP:conf/colt/AgrawalAGZ17}, i.e.,  $\mathbb{P}(\mathcal{A}_{\ell-1})\leq \frac{1}{\ell^2}$.
By Proposition \ref{pro:TSlipschitz}, we have 
\begin{equation*}
    \Delta R_\ell\leq (1+V(S_\ell))\frac{\tilde{\gamma}_2}{1+\sum_{i\in S_\ell}\max\{\tilde{v}^\ell_i ,v_i\}}\sum_{i\in S_\ell}\vert \tilde{v}^\ell_i -v_i\vert\leq \tilde{\gamma}_2 \sum_{i\in S_\ell}\vert \tilde{v}^\ell_i -v_i\vert.
\end{equation*}
Therefore, 
\begin{align*}
    \mathbb{E}[\Delta R_\ell \ind{\mathcal{A}^c_{\ell-1}}] &\leq \mathbb{E}\left[\tilde{\gamma}_2 \sum_{i\in S_\ell}\vert \tilde{v}^\ell_i -v_i\vert \ind{\mathcal{A}^c_{\ell-1}} \right]\\
    &\leq \mathbb{E} \left[\tilde{\gamma}_2 \sum_{i\in S_\ell}\vert \tilde{v}^\ell_i -\bar{v}_i(\ell)\vert \ind{\mathcal{A}^c_{\ell-1}} \right]+\mathbb{E} \left[\tilde{\gamma}_2 \sum_{i\in S_\ell}\vert \bar{v}_i(\ell)-v_i\vert \ind{\mathcal{A}^c_{\ell-1}} \right]\\
    &\leq \tilde{\gamma}_2\mathbb{E} \left[ \sum_{i\in S_\ell}\vert \tilde{v}^\ell_i -\bar{v}_i(\ell)\vert  \right]+\tilde{\gamma}_2\mathbb{E} \left[\sum_{i\in S_\ell}\left(\sqrt{\frac{24v_i\log(\ell+1)}{T_i(\ell)}}+\frac{48\log(\ell+1)}{T_i(\ell)} \right)\right],
\end{align*}
where the last inequality follows from the definition of $\mathcal{A}_{\ell-1}$. By the definition of $\tilde{v}^\ell_i $ and $\bar{v}_i(\ell)$, there is 
\begin{equation*}
    \mathbb{E} \left[ \sum_{i\in S_\ell}\vert \tilde{v}^\ell_i -\bar{v}_i(\ell)\vert  \right]=\mathbb{E} \left[ \sum_{i\in S_\ell}\left\vert \max_{j=1,\cdots,K }\{\theta^{(j)}_\ell  \}\right\vert \hat{\sigma}_i(\ell)  \right]=\mathbb{E} \left[ \sum_{i\in S_\ell} \mathbb{E}\left[\left\vert \max_{j=1,\cdots,K }\{\theta^{(j)}_\ell  \}\right\vert\right] \hat{\sigma}_i(\ell)  \right],
\end{equation*}
where $\theta^{(j)}_\ell$ are independent standard normal distributed random variables given $\ell$. By Theorem~1 in \cite{Kamath20}, it is easy to verify that
\begin{equation*}
    \mathbb{E}\left[\left\vert \max_{j=1,\cdots,K }\{\theta^{(j)}_\ell  \}\right\vert\right]\leq 4\sqrt{\log K}.
\end{equation*}
Since $\ell\leq T$, by the definition of $\hat{\sigma}_i(\ell)$, we have
\begin{equation}\label{eq:reg2eq3}
    \mathbb{E}[\Delta R_\ell \ind{\mathcal{A}^c_{\ell-1}}] \lesssim  \mathbb{E}\left[ \sum_{i\in S_\ell}\sqrt{\frac{v_i \log (TK)}{T_i(\ell)}} \right]+ \mathbb{E}\left[ \sum_{i\in S_\ell}\frac{ \log (TK)}{T_i(\ell)} \right] .
\end{equation}
Combining \eqref{eq:reg2eq1}, \eqref{eq:reg2eq2} and \eqref{eq:reg2eq3} gives
\begin{equation}\label{eq:reg2eq4}
    \reg_T^2 \lesssim K \mathbb{E}\left[\sum_{\ell=1}^L \frac{1}{\ell^2}\right] + \log (TK) \mathbb{E}\left[ \sum_{\ell=1}^L\sum_{i\in S_\ell}\sqrt{\frac{v_i }{T_i(\ell)}} \right]+ \log (TK) \mathbb{E}\left[ \sum_{\ell=1}^L\sum_{i\in S_\ell}\frac{1 }{T_i(\ell)} \right] .
\end{equation}
Finally we apply the argument for $(*)$, $(**)$ and $(***)$ in the proof of Theorem~\ref{thm:main} to \eqref{eq:reg2eq4} and we obtain $\reg_T^2 = \tilde{\mathcal{O}}( \sqrt{NT}) $.

{\em Bounding $\reg_T^1$}: It remains to show that $\reg_T^1\leq \tilde{\mathcal{O}}( \sqrt{NT}) $. Notice that 
\begin{equation}\label{eq:reg2eq5}
\begin{aligned}
    \reg_T^1 &=\mathbb{E}\left[ \sum_{\ell=1}^L \ind{\ell\in\mathcal{T}}  \sum_{\tau\in \mathcal{E}^{An}(\ell)} l_\tau [ U (F(S^*,\bm{v})) - U (F(S_{\tau},\tilde{\bm{v}}^\tau))]\right]\\
    &\leq \mathbb{E}\left[ \sum_{\ell=1}^L \ind{\ell\in\mathcal{T}} \sum_{\tau\in \mathcal{E}^{An}(\ell)} l_\tau [ U (F(S_{\ell},\tilde{\bm{v}}^\ell )) - U (F(S_{\tau},\tilde{\bm{v}}^\tau))]\right]\\
    &\leq \mathbb{E}\left[ \sum_{\ell=1}^L \ind{\ell\in\mathcal{T}} \sum_{\tau\in \mathcal{E}^{An}(\ell)} \Delta R_{\ell,\tau}\right],
\end{aligned}
\end{equation}
where the first inequality follows from Lemma \ref{lem:monomax} and the second inequality follows from the optimality of $S_\tau$ under parameter $\tilde{\bm{v}}^\tau$. Similar to \eqref{eq:reg2eq2}, we bound $\Delta R_{\ell,\tau}$ in two scenarios 
\begin{equation}\label{eq:reg2eq6}
    \begin{aligned}
    \mathbb{E}\left[ \sum_{\tau\in \mathcal{E}^{An}(\ell)} \Delta R_{\ell,\tau} \right] &=\mathbb{E}\left[ \sum_{\tau\in \mathcal{E}^{An}(\ell)} \Delta R_{\ell,\tau} \ind{\mathcal{A}_{\ell-1}}+\Delta R_{\ell,\tau} \ind{\mathcal{A}^c_{\ell-1}}  \right]\\
    &\leq 2\gamma_1(K+1)\mathbb{E}[\vert \mathcal{E}^{An}(\ell)\vert \ind{\mathcal{A}_{\ell-1}} ]+\mathbb{E}\left[ \sum_{\tau\in \mathcal{E}^{An}(\ell)} \Delta R_{\ell,\tau} \ind{\mathcal{A}^c_{\ell-1}}  \right].
    \end{aligned}
\end{equation}
By \cite[Lemma 5]{DBLP:conf/colt/AgrawalAGZ17}, i.e., 
\begin{equation} \label{ts-key-lemma}
    \left[\mathbb{E}\left(\vert \mathcal{E}^{An}(\ell)\vert^2\right)\right]^{1/2}\leq \frac{e^{12}}{K}+\sqrt{30},
\end{equation}
and \cite[Lemma 7]{DBLP:conf/colt/AgrawalAGZ17}, i.e.,  $\mathbb{P}(\mathcal{A}_{\ell-1})\leq \frac{1}{\ell^2}$, we have 
\begin{equation}
    \begin{aligned}
    (K+1)\mathbb{E}[\vert \mathcal{E}^{An}(\ell)\vert \ind{\mathcal{A}_{\ell-1}} ]\leq (K+1)\left[\mathbb{E}\left(\vert \mathcal{E}^{An}(\ell)\vert^2\right)\right]^{1/2}\left[\mathbb{P}(\mathcal{A}_{\ell-1})\right]^\frac{1}{2} \lesssim  \frac{(K+1)}{\ell}.
    \end{aligned}
\end{equation}
For the second term in \eqref{eq:reg2eq6}, by Proposition \ref{pro:TSlipschitz}
\begin{align*}
    \Delta R_{\ell,\tau} &\leq (1+V(S_\tau))\frac{\tilde{\gamma}_2}{1+\sum_{i\in S_\ell}\max\{\tilde{v}^\ell_i ,\tilde{v}^\tau_i\}}\sum_{i\in S_\ell}\vert \tilde{v}^\ell_i -\tilde{v}^\tau_i\vert\\
    &\leq \frac{\tilde{\gamma}_2 (K+1)}{1+\sum_{i\in S_\ell}\tilde{v}^\ell_i}\sum_{i\in S_\ell}\vert \tilde{v}^\ell_i -\tilde{v}^\tau_i\vert\\
    &\leq \frac{\tilde{\gamma}_2 (K+1)}{1+V(S_\ell)}\sum_{i\in S_\ell}(\vert \tilde{v}^\ell_i -v_i\vert+\vert \tilde{v}^\tau_i -v_i\vert),
\end{align*}
where the last inequality follows from $\tilde{v}^\ell_i \geq v_i$ because $\ell$ is an optimistic episode. Then
\begin{align*}
    \mathbb{E}\left[ \sum_{\tau\in \mathcal{E}^{An}(\ell)} \Delta R_{\ell,\tau} \ind{\mathcal{A}^c_{\ell-1}}  \right] & \lesssim  (K+1) \mathbb{E}\left[ \sum_{\tau\in \mathcal{E}^{An}(\ell)}\frac{\ind{\mathcal{A}^c_{\ell-1}}}{1+V(S_\ell)}\sum_{i\in S_\ell}(\vert \tilde{v}^\ell_i -v_i\vert+\vert \tilde{v}^\tau_i -v_i\vert) \right]  \\
    & \lesssim (K+1) \mathbb{E}\left[ \frac{\ind{\mathcal{A}^c_{\ell-1}}}{1+V(S_\ell)}\sum_{\tau\in \mathcal{E}^{An}(\ell)}\sum_{i\in S_\ell}\vert \tilde{v}^\ell_i -v_i\vert \right] \\
    &\quad+ (K+1) \mathbb{E}\left[ \frac{\ind{\mathcal{A}^c_{\tau-1}}}{1+V(S_\ell)}\sum_{\tau\in \mathcal{E}^{An}(\ell)}\sum_{i\in S_\ell}\vert \tilde{v}^\tau_i -v_i\vert \right]  \\
    &\quad+ (K+1) \mathbb{E}\left[ \frac{\ind{\mathcal{A}_{\tau-1}}}{1+V(S_\ell)}\sum_{\tau\in \mathcal{E}^{An}(\ell)}\sum_{i\in S_\ell}\vert \tilde{v}^\tau_i -v_i\vert \right].
\end{align*}
We can bound the first two terms in the same way as we obtain \eqref{eq:reg2eq3}, and we can bound the third term by \cite[Lemma 7]{DBLP:conf/colt/AgrawalAGZ17} to obtain 
\begin{align*}
    \mathbb{E}\left[ \sum_{\tau\in \mathcal{E}^{An}(\ell)} \Delta R_{\ell,\tau} \ind{\mathcal{A}^c_{\ell-1}}  \right] & \lesssim  (K+1) \mathbb{E}\left[\frac{\vert \mathcal{E}^{An}(\ell)\vert}{1+V(S_\ell)} \sum_{i\in S_\ell}\left(\sqrt{\frac{v_i \log (TK)}{T_i(\ell)}}+\frac{ \log (TK)}{T_i(\ell)}+\frac{1}{\ell^2}\right) \right].
\end{align*}
Apply the Cauchy–Schwarz inequality and we have 
\begin{align*}
    \reg_T^1 & \lesssim \mathbb{E}\left[\sum_{\ell=1}^L \frac{(K+1)}{\ell^2}\right] + (K+1) \mathbb{E}\left[ \sum_{\ell=1}^L\frac{\vert \mathcal{E}^{An}(\ell)\vert}{1+V(S_\ell)} \sum_{i\in S_\ell}\left(\sqrt{\frac{v_i \log TK}{T_i(\ell)}}+\frac{ \log TK}{T_i(\ell)}+\frac{1}{\ell^2}\right) \right]\\
    &\lesssim (K+1)\left[1+\left( \mathbb{E}\left[\sum_{\ell=1}^L\vert \mathcal{E}^{An}(\ell)\vert^2\right]\right)^{1/2}\left(\left( \mathbb{E}\left[\sum_{\ell=1}^L \delta^2(\ell)\right]\right)^{1/2}+\left( \mathbb{E}\left[\sum_{\ell=1}^L \Delta^2(\ell)\right]\right)^{1/2}+\sqrt{K}\right)\right].
\end{align*}
where 
\begin{align*}
    \delta(\ell)=\frac{1}{1+V(S_\ell)}\sum_{i\in S_\ell}\sqrt{\frac{v_i \log (TK)}{T_i(\ell)}},\qquad
    \Delta(\ell)=\frac{1}{1+V(S_\ell)}\sum_{i\in S_\ell}\frac{ \log (TK) }{T_i(\ell)}.
\end{align*}
The bound of $\mathcal{E}^{An}(\ell)$, $\delta(\ell)$ and $\Delta(\ell)$ follows the same argument as in \cite[Proof of Theorem 1 and Lemma 5]{DBLP:conf/colt/AgrawalAGZ17}, and we can conclude that 
\begin{equation*}
    \reg_T^1 = \tilde{\mathcal{O}}( \sqrt{NT}) .
\end{equation*}
\end{proof}

\end{document}